\documentclass[letterpaper]{article}
\usepackage{aaai}
\usepackage{times} 
\usepackage{helvet} 
\usepackage{courier} 
\usepackage{graphicx}
\usepackage{subfig}
\usepackage{url}
\usepackage{amssymb}
\usepackage{comment}
\usepackage{color}
\usepackage{xspace}
\usepackage{framed}

\newcommand{\argmax}{\text{argmax}}
\newcommand{\citet}[1]{\citeauthor{#1} (\citeyear{#1})}

\newcommand{\AlgRG}{{\textsc{Stochastic-Greedy}}\xspace}
\newcommand{\AlgMG}{{\textsc{Multi-Greedy}}\xspace}
\newcommand{\AlgSG}{{\textsc{Sample-Greedy}}\xspace}
\newcommand{\AlgRS}{{\textsc{Random-Selection}}\xspace}
\newcommand{\AlgTG}{{\textsc{Threshold-Greedy}}\xspace}
\newcommand{\AlgFG}{{\textsc{Threshold-Greedy}}\xspace}
\newcommand{\AlgLG}{\textsc{Lazy-Greedy}\xspace}
\newcommand{\kr}[2]{\mathcal{K}_{e_{#1},e_{#2}}}

\usepackage{amsmath}
\usepackage{algorithm}
\usepackage{amsthm}
\usepackage[noend]{algorithmic}

\newcommand{\E}{\mbox{\bf E}}

\newcommand{\INPUT}{\item[{\bf Input:}]}
\newcommand{\OUTPUT}{\item[{\bf Output:}]}
\newcommand{\RR}{{\mathbb R}}

\newtheorem{theorem}{Theorem}
\newtheorem{lemma}[theorem]{Lemma}
\title{Lazier Than Lazy Greedy}
\author{
Baharan Mirzasoleiman \\ ETH Zurich \\ baharanm@inf.ethz.ch \\
\And Ashwinkumar Badanidiyuru \\ Google Research Mountain View \\ ashwinkumarbv@google.com \\
\And Amin Karbasi \\ Yale University \\ amin.karbasi@yale.edu \\
\AND
\And Jan Vondr{\'{a}}k \\ IBM Almaden \\ jvondrak@us.ibm.com \\
\And Andreas Krause \\ ETH Zurich \\ krausea@ethz.ch \\
\And
}

\begin{document}
\maketitle

\begin{abstract}
Is it possible to maximize a monotone submodular function faster than 
the widely used  lazy greedy algorithm (also known as accelerated greedy), both in theory and  practice? 
In this paper, we develop  the first linear-time algorithm for maximizing a general monotone submodular function subject to a cardinality constraint. We  show that our randomized algorithm, \AlgRG, can  achieve a $(1-1/e-\varepsilon)$ approximation guarantee, in expectation, to the optimum solution in time \textit{linear} in the size of the data and \textit{independent} of the cardinality constraint. 
%We then provide a deterministic algorithm, \AlgAG,  that aggressively tries to maximize the value of the submodular function while achieving a $1/2$ approximation guarantee to the optimum solution. 
 We empirically demonstrate the effectiveness of our algorithm on submodular functions arising in data summarization, including training large-scale kernel methods, exemplar-based clustering, and sensor placement. We observe that  \AlgRG  practically  achieves the same  utility value as  lazy greedy but runs much faster. More surprisingly, we observe that in many practical scenarios \AlgRG  does not  evaluate the whole fraction of data points even once and still achieves  indistinguishable results compared to lazy greedy. 
\end{abstract}

\section{Introduction}\label{sec:intro}

For the last several years, we have witnessed the emergence of datasets  of an unprecedented scale across  different scientific disciplines. The large volume of such datasets presents new computational challenges as the diverse, feature-rich, unstructured and usually high-resolution data does not allow for effective data-intensive inference. In this regard, \textit{data summarization} is a compelling  (and sometimes the only) approach that aims at both exploiting the richness of large-scale data and being computationally tractable. Instead of operating on complex and large  data directly, carefully constructed summaries not only enable the execution of various data analytics tasks but also  improve their efficiency and scalability.

In order to effectively summarize the data, we need to define a measure for the amount of \textit{representativeness} that lies within a selected set. If we think of representative elements as the ones that cover best, or are most informative w.r.t.~the items in a dataset then naturally adding a new element to a  set of representatives, say $A$, is more beneficial than adding it to its superset, say $B\supseteq A$, as the new element can potentially enclose more uncovered items when considered with elements in $A$ rather than $B$. This intuitive \textit{diminishing returns} property  can be systematically formalized through \textit{submodularity} (c.f., \citet{nemhauser78}). More precisely, a submodular function $f: 2^V\rightarrow \mathbb{R}$ assigns a subset $A \subseteq V$ a utility value $f(A)$ --measuring the representativeness of the set $A$-- such that
%
% is a set function from the subsets of the ground set $V$ to real numbers that satisfies  
 $$f(A\cup\{i\}) - f(A) \geq f(B\cup\{i\}) - f(B)$$
for any $A\subseteq B\subseteq V$ and  $i \in V \setminus B$. Note that $\Delta(i|A) \doteq f(A\cup\{i\}) - f(A)$ measures the marginal \textit{gain} of adding a new element $i$ to  a summary $A$. Of course, the meaning of representativeness (or utility value)  depends very much on the underlying application; for a collection of random variables, the utility of a subset can be measured in terms of entropy, and for a collection of vectors, the utility of a subset can be measured in terms of the dimension of a subspace spanned by them. In fact, summarization through submodular functions has gained a lot of interest in recent years with application ranging from  exemplar-based clustering \cite{gomes10budgeted}, 
to document {\cite{lin11class,DasguptaKR13} and corpus summarization \cite{sipos12temporal}, to recommender systems \cite{leskovec07,elarini09turning,elarini11beyond}.

Since we would like to choose a summary of a manageable size, a natural optimization problem is to find a summary $A^*$ of size at most $k$ that maximizes the utility, i.e., 
\begin{equation}\label{problem}
A^* = \argmax_{A: |A|\leq k} f(A). 
\end{equation}

Unfortunately,  this optimization problem is NP-hard for many classes of submodular functions \cite{nemhauser78best,feige98threshold}. We say a submodular function is {\em monotone}  if for any $A\subseteq B\subseteq V$ we have $f(A)\leq f(B)$. A celebrated result   of \citet{nemhauser78} --with great importance in artificial intelligence and machine learning-- states that for non-negative monotone submodular functions a simple \textit{greedy algorithm} provides a solution with $(1-1/e)$ approximation guarantee  to the  optimal (intractable) solution. This greedy algorithm starts with the empty set $A_0$ and in iteration $i$, adds an element maximizing the marginal gain $\Delta(e|A_{i-1})$. For a ground set $V$ of size $n$, this greedy algorithm needs $O(n\cdot k)$ function evaluations in order to find a summarization of size $k$. However, in many data intensive applications, evaluating $f$ is expensive and  running the standard greedy algorithm is infeasible. Fortunately, submodularity can be exploited to implement an accelerated version of the  classical greedy algorithm, usually called \AlgLG \cite{minoux78accelerated}. Instead of computing $\Delta(e|A_{i-1})$ for each element $e\in V$, the \AlgLG algorithm keeps an upper bound $\rho(e)$ (initially $\infty$) on the marginal gain sorted in decreasing order. In each iteration $i$, the \AlgLG algorithm evaluates the element on top of the list, say $e$, and updates its upper bound, $\rho(e)\leftarrow \Delta(e|A_{i-1})$. If after the update $\rho(e)\geq \rho(e')$ for all $e'\neq e$, submodularity guarantees that $e$ is the element with the largest marginal gain. Even though the exact cost (i.e., number of function evaluations) of \AlgLG is unknown, this algorithm leads to orders of magnitude  speedups in practice. As a result, it has been used as the state-of-the-art implementation in numerous applications including  network monitoring \cite{leskovec07}, network inference \cite{gomezrodriguez12inferring},  document summarization \cite{lin11class}, and speech data subset selection \cite{wei2013using}, to name a few. However, as the size of the data increases, even for small  values of $k$, running \AlgLG is infeasible. A natural question to ask is whether it is possible to further accelerate \AlgLG by a procedure with a weaker dependency on $k$. Or even better,  is it possible to have an algorithm that does not depend on $k$ at all and scales linearly with the data size $n$?
%
%This is due to the fact that inherently the cost of lazy greedy increases as we increase $k$.  For further information, look at a recent survey by \cite{krause12survey}. 

%
%In fact, if we assume nothing but submodularity, no efficient algorithm produces better solutions in general \cite{nemhauser78best,feige98threshold}. 
In this paper, we propose the first linear-time algorithm, \AlgRG,  for maximizing a non-negative monotone submodular function subject to a cardinality constraint $k$. We show that \AlgRG achieves a $(1-1/e-\epsilon)$ approximation guarantee to the optimum solution with running time $O(n\log(1/\epsilon))$ (measured in terms of function evaluations) that is \textit{independent} of $k$. Our experimental results on exemplar-based clustering and  active set selection in nonparametric learning also confirms that \AlgRG consistently outperforms \AlgLG by a large margin while achieving practically the same utility value. More surprisingly, in our experiments we observe that \AlgRG sometimes does not even evaluate all the items and shows a running time that is  less than $n$  while still providing  solutions close to the ones returned by \AlgLG. Due to its independence of $k$, \AlgRG is the first algorithm that truly scales to voluminous datasets. 

\section{Related Work}\label{sec:related}

Submodularity is a property of set functions with deep theoretical and practical consequences. For instance, submodular maximization generalizes many well-known combinatorial problems including maximum weighted matching, max coverage, and facility location, to name a few. It has also found numerous applications in machine learning and artificial intelligence such as  influence maximization \cite{kempe03}, information gathering~\cite{krause11submodularity}, document summarization~\cite{lin11class} and active learning~\cite{guillory2011-active-semisupervised-submodular,golovin11jair}.
%
%Submodular maximization generalizes many well-known problems, e.g., maximum weighted matching, max coverage, and finds numerous applications in machine learning and social networks, such as influence maximization \cite{kempe03}, information gathering~\cite{krause11submodularity}, document summarization~\cite{lin11class} and active learning~\cite{guillory2011-active-semisupervised-submodular,golovin11jair}.
%
%Over the recent years, submodular optimization has been identified as a powerful tool for numerous data mining and machine learning applications including  viral marketing \cite{kempe03},  network monitoring \cite{leskovec07},  news article recommendation \cite{elarini09turning},  nonparametric learning \cite{gomes10budgeted,reed13scaling},  document and corpus summarization \cite{lin11class,DasguptaKR13,sipos12temporal},  network inference \cite{gomezrodriguez12inferring}, and Determinantal Point Processes \cite{GillenwaterKT12}.  A problem of key importance in all these applications is to maximize a monotone submodular function subject to a cardinality constraint (i.e., a bound on the number $k$ of elements that can be selected). See \cite{krause12survey} for a survey on submodular maximization.
In most of these applications one needs  to handle increasingly larger quantities of data. For this purpose, accelerated/lazy variants~\cite{minoux78accelerated,leskovec07} of the celebrated greedy algorithm of~\citet{nemhauser78} have been extensively used. 

\vspace{-3mm}
\paragraph{Scaling Up:}\looseness -1  To solve the optimization problem \eqref{problem} at scale, there have been very recent efforts to either make use of parallel computing methods or treat data in a streaming fashion. In particular, \citet{Chierichetti2010} and \citet{blelloch11}  addressed a particular instance of submodular functions, namely, maximum coverage  and provided a distributed method with a constant factor approximation to the centralized algorithm. More generally, \citet{kumar13fast} provided a constant approximation guarantee for general submodular functions with bounded marginal gains. Contemporarily, \citet{mirzasoleiman13distributed} developed a two-stage distributed algorithm that guarantees solutions close to the optimum if the dataset is massive and the submodular function is smooth. 

Similarly, \citet{gomes10budgeted} presented a heuristic streaming algorithm for submodular function maximization and showed that under strong assumptions about the way the data stream is generated their method is effective. Very recently, \citet{badanidiyuru14streaming} provided the first one-pass streaming algorithm with a constant factor approximation guarantee for general submodular functions without any assumption about the data stream. 

Even though the goal of this paper is quite different and complementary in nature  to the aforementioned  work, \AlgRG can be easily integrated into existing distributed methods. For instance, \AlgRG can replace \AlgLG for solving each sub-problem in the approach of \citet{mirzasoleiman13distributed}. More generally, any distributed algorithm that uses \AlgLG as a sub-routine, can directly benefit from our method and  provide even more efficient large-scale algorithmic frameworks.

\vspace{-3mm}
\paragraph{Our Approach:} In this paper, we develop the first centralized algorithm whose cost (i.e., number of function evaluations) is independent of the cardinality constraint, which in turn directly addresses the shortcoming of \AlgLG. Perhaps the closest, in spirit, to our efforts are approaches by \citet{wei2014fast} and \citet{bv14soda}. Concretely, \citet{wei2014fast} proposed a multistage algorithm, \AlgMG,  that tries to decrease the running time of \AlgLG by approximating the underlying submodular function with a set of (sub)modular functions that can be potentially evaluated less expensively. This approach is  effective only for those submodular functions that can be easily decomposed and approximated. Note again that \AlgRG can be used for solving the sub-problems in each stage of \AlgMG to develop a faster multistage method. %Moreover, by replacing a submodular function with a set of simpler modular functions, \AlgMG potentially increases the number of function evaluations (by a factor of 2) with the hope that evaluating such functions would be less time-consuming.  
\citet{bv14soda} proposed a different  centralized algorithm that achieves a $(1-1/e-\epsilon)$ approximation guarantee for general submodular functions using  $O(n/\epsilon\log(n/\epsilon))$ function evaluations. 
However, \AlgRG consistently outperforms their algorithm in practice in terms of cost, and returns higher utility value. In addition, \AlgRG uses only $O(n\log(1/\epsilon))$ function evaluations  in theory, and thus provides a stronger analytical guarantee.

\section{\AlgRG Algorithm}
In this section, we present our randomized greedy algorithm \AlgRG and then show how to combine it with lazy evaluations. We will show that \AlgRG has provably linear running time independent of $k$, while simultaneously having the same %worst case
 approximation ratio guarantee (in expectation). In the following section we will further demonstrate through experiments that this is also reflected in practice, i.e., \AlgRG is substantially faster than \AlgLG, while being practically identical to it in terms of the utility. 

The main idea behind \AlgRG is to produce an element which improves the value of the solution roughly the same as greedy, but in a fast manner. This is achieved by a sub-sampling step. At a very high level this is similar to how stochastic gradient descent improves the running time of gradient descent for convex optimization. 

\subsection{Random Sampling}
%\looseness -1 
The key reason that the classic greedy algorithm  works is that at each iteration $i$, an element is identified that reduces the gap to the optimal solution by a significant amount, i.e., by at least $(f(A^*)-f(A_{i-1}))/k$. 
%at each step it adds at least $1/k$ fraction of the remaining value, i.e $(f(A^*)-f(A))/k$.
%The reason the classic greedy algorithm  works is that at each step it adds at least $1/k$ fraction of the remaining value, i.e $(f(A^*)-f(A))/k$. 
This requires $n$ oracle calls per step, the main bottleneck of the classic greedy algorithm. Our main observation here is that by submodularity, we can achieve the same improvement by adding a uniformly random element from $A^*$ to our current set $A$. To get  this improvement, we will see that it is enough to randomly sample a set $R$ of size $(n/k)\log(1/\epsilon)$, which in turn overlaps with $A^*$ with probability $1-\epsilon$.  
%and we hit the optimal solution (i.e $R\cap A^*\neq \emptyset$) with probability $1-\epsilon$. 
This is the main reason we are able to achieve a boost in performance. 

The algorithm is formally presented in Algorithm~\ref{alg:linear-cardinality}. Similar to the greedy algorithm, our algorithm starts with an empty set and adds one element at each iteration. But in each step it first samples a set $R$ of size $(n/k)\log(1/\epsilon)$ uniformly at random and then adds the element from $R$ to $A$ which increases its value the most. 

\begin{algorithm}
\caption{\AlgRG}
\label{alg:linear-cardinality}
\begin{algorithmic}[1]
\INPUT $f:2^V \rightarrow \RR_+$, $k \in \{1,\ldots,n\}$.
\OUTPUT A set $A \subseteq V$ satisfying $|A| \leq k$.
\STATE $A \leftarrow \emptyset$.
\FOR {($i \leftarrow 1;~ i \le k;~ i \leftarrow i+1)$}
\STATE  $R \leftarrow$ a random subset obtained by sampling $s$ random elements from $V \setminus A$.
\STATE  $a_i \leftarrow \argmax_{a \in R} \Delta(a|A)$.
\STATE  $A \leftarrow A \cup \{a_i\}$ 
\ENDFOR
\RETURN $A$.
\end{algorithmic}
\end{algorithm}
\looseness -1 Our main theoretical result is the following. It shows that \AlgRG achieves a   near-optimal solution for general monotone submodular functions,
with 
%having a 
computational complexity independent of the cardinality constraint. 
\begin{theorem}
\label{thm:cardinality}
Let $f$ be a non-negative monotone submoduar function. Let us also set $s=\frac{n}{k} \log \frac{1}{\epsilon}$. Then \AlgRG achieves a $(1-1/e-\epsilon)$ approximation guarantee in expectation to the optimum solution of problem \eqref{problem} with only $O(n \log \frac{1}{\epsilon})$ function evaluations. 
%If you set $s_i=\frac{n}{k} \log \frac{1}{\epsilon}$ then algorithm~\ref{alg:linear-cardinality} runs in time $O(n \log \frac{1}{\epsilon})$ and provides a $(1-1/e-\epsilon)$-approximation in expectation for the problem $\max \{f(A): |A| \leq k\}$ where $f:2^V \rightarrow \RR_+$ is monotone submodular and $1 \leq k \leq n = |V|$.
\end{theorem}
Since there are $k$ iterations in total and at each iteration we have $(n/k)\log(1/\epsilon)$ elements, the total number of function evaluations cannot be more than $k\times (n/k)\log(1/\epsilon) = n\log(1/\epsilon)$.  The proof of the approximation guarantee  is given in the analysis section. 

\subsection{Random Sampling with Lazy Evaluation}
While our theoretical results show a provably linear time algorithm, we can combine the random sampling procedure with lazy evaluation to boost its performance. There are mainly two reasons why lazy evaluation helps. First, the  randomly sampled sets can overlap and we can exploit the previously evaluated marginal gains. Second, as in  \AlgLG  although the marginal values of the elements might change in each step of the greedy algorithm, often their ordering does not change \cite{minoux78accelerated}. Hence in line 4 of Algorithm~\ref{alg:linear-cardinality} we can apply directly  lazy evaluation as follows.  We maintain an upper bound $\rho(e)$ (initially $\infty$) on the marginal gain of all elements sorted in decreasing order. In each iteration $i$,  \AlgRG samples a set $R$. From this set $R$ it evaluates the element that comes on top of the list. Let's denote this element by $e$. It then updates the upper bound for $e$, i.e.,  $\rho(e)\leftarrow \Delta(e|A_{i-1})$.
%It then evaluates the element on top of the list from the set $R$, say $e$,  and updates its upper bound, $\rho(e)\leftarrow \Delta(e|A_{i-1})$. 
If after the update $\rho(e)\geq \rho(e')$ for all $e'\neq e$ where $e,e'\in R$, submodularity guarantees that $e$ is the element with the largest marginal gain in the set $R$. Hence, lazy evaluation helps us reduce function evaluation in each round. 

\vspace{-.2cm}
\section{Experimental Results}\label{sec:experiments}
In this section, we address the following questions:
1) how well does \AlgRG perform compared to previous art and in particular \AlgLG, and 2) How does \AlgRG help us get near optimal solutions on large datasets by reducing the computational complexity? To this end, we compare the performance of our \AlgRG method to the
following benchmarks: \textsc{Random-Selection}, where the output is $k$ randomly selected data points from $V$; \AlgLG, where the output is the $k$ data points produced by the accelerated greedy method \cite{minoux78accelerated}; \textsc{Sample-Greedy}, where the output is the $k$ data points produced by applying \AlgLG on a subset of data points  parametrized by sampling probability $p$; and \AlgFG, where the output is the $k$ data points provided by  the algorithm of \citet{bv14soda}. %; and \AlgMG, where the output is the $k$ data point provided by the algorithm of \citet{wei2014fast}. 
In order to compare the computational cost of different methods independently of the concrete implementation and platform, in our experiments we measure the computational cost in terms of the number of function evaluations used. Moreover, to implement the \AlgSG method, random subsamples are generated geometrically using different values for probability $p$. Higher values of $p$ result in subsamples of larger size from the original dataset. 
%We also report the  best results obtained for \AlgFG and \AlgMG in terms of function evaluations. 
To maximize fairness, we implemented an accelerated version of \AlgFG with lazy evaluations (not specified in the paper) and report the  best results in terms of function evaluations. 
%In particular, we do not implement the pruning step in  \AlgMG as it increases the cost by at least a factor of two. 
Among all benchmarks, \AlgRS has the lowest computational cost (namely, one) as we need to only evaluate the selected set at the end of the sampling process. However, it provides the lowest utility. On the other side of the spectrum,  \AlgLG makes $k$ passes over the full ground set, providing typically the best solution in terms of utility. The lazy evaluation eliminates a large fraction of the function evaluations in each pass. Nonetheless, it is still computationally prohibitive for large values of $k$.

In our experimental setup, we focus on three important and classic machine learning applications: nonparametric learning, exemplar-based clustering, and sensor placement.
%\begin{itemize}
%\item
%\textsc{Random-Selection:} the output is $k$ randomly selected data points from $V$.
%%\item 
%%\textbf{Standard Greedy:} the output is the $k$ data points selected by the standard greedy algorithm. This algorithm is not applicable on large datasets.
%\item
%\AlgLG: the output is the $k$ data points produced by the accelerated greedy method \cite{minoux78}. 
%\item
%\textsc{Sample-Greedy}: the output is the $k$ data points produced by applying \AlgLG on a subset of data points  parametrized by probability $p$.
%
% %The standard greedy algorithm is run over a subsample of data elements to find a solution of size $k$.
%\item
%\AlgFG: the output is the $k$ data points provided by \AlgFG  \cite{badanidiyuru14streaming}. 
%%It maintains a continuously decreasing threshold where for each threshold, \AlgFG makes a pass over the entire dataset and includes any element with a marginal value above the threshold.
%\end{itemize}
%%\subsection{Example Applications}
\vspace{-.2cm}
\paragraph{Nonparametric Learning.}
Our first application is data subset selection in nonparametric learning. We focus on the special case of Gaussian Processes (GPs) below, but similar problems arise in large-scale kernelized SVMs and other kernel machines.
Let $\mathbf{X}_{V}$, be a set of random variables indexed by  the ground set $V$. In a Gaussian Process (GP) we assume that  every  subset $\mathbf{X}_S$, for $S=\{e_1, \dots, e_s\}$, is distributed according to a multivariate normal distribution, i.e., 
$P(\mathbf{X}_S = \mathbf{x}_S) = \mathcal{N}(\mathbf{x}_S;\mu_S, \Sigma_{S,S})$,
 where   $ \mu_S = (\mu_{e_1}, \dots, \mu_{e_s})$ and $\Sigma_{S,S} = [\kr{i}{j}] (1\leq i,j\leq k)$
are the prior mean vector and  prior covariance matrix, respectively. 
%Formally, a Gaussian Process (GP) is a joint probability distribution over a set of random variables $\mathbf{X}_{V}$, indexed by  the ground set $V$, such that every finite subset $\mathbf{X}_S$, for $S=\{e_1, \dots, e_s\}$, is distributed according to a multivariate normal distribution, i.e., $P(\mathbf{X}_S = \mathbf{x}_S) = \mathcal{N}(\mathbf{x}_S;\mu_S, \Sigma_{S,S})$, where   $ \mu_S = (\mu_{e_1}, \dots, \mu_{e_s})$ and $\Sigma_{S,S} = [\kr{i}{j}] (1\leq i,j\leq k)$
%are the prior mean vector and  prior covariance matrix, respectively.  
The covariance matrix is given in terms of a positive definite kernel $\mathcal{K}$, e.g., the squared exponential kernel $\kr{i}{j} = \exp(-|e_i-e_j|^2_2/h^2)$ is a common choice in practice.
%
%A commonly used kernel function in practice where elements of the ground set $V$ are embedded in a Euclidean space is the squared exponential kernel $\kr{i}{j} = \exp(-|e_i-e_j|^2_2/h^2)$. 
In GP regression, each data point $e\in V$ is considered a random variable. Upon observations $\mathbf{y}_A = \mathbf{x}_A+\mathbf{n}_A$ (where $\mathbf{n}_A$ is a vector of independent Gaussian noise with variance $\sigma^2$), the predictive distribution of a new data point $e\in V$ is a normal distribution $P(\mathbf{X}_e\mid \mathbf{y}_A) = \mathcal{N}(\mu_{e|A}, \Sigma^2_{e|A})$, where 
\begin{align}\label{eq:GP}
&\mu_{e|A} = \mu_e + \Sigma_{e,A} (\Sigma_{A,A}+\sigma^2\mathbf{I})^{-1} (\mathbf{x}_A-\mu_A), \nonumber \\
&\sigma^2_{e|A} = \sigma^2_e - \Sigma_{e,A} (\Sigma_{A,A}+\sigma^2\mathbf{I})^{-1}\Sigma_{A,e}.
\end{align}
Note that evaluating \eqref{eq:GP} is computationally expensive as it requires a matrix inversion. Instead, most efficient approaches for making predictions in GPs rely on choosing a small -- so called \textit{active} -- set of data points. For instance, in the Informative Vector Machine (IVM) we seek a summary $A$ such that the information gain, $f(A) = I (\mathbf{Y}_A;\mathbf{X}_{V}) = H(\mathbf{X}_{V}) - H(\mathbf{X}_{V}|\mathbf{Y}_{A})=\frac{1}{2}\log\det(\mathbf{I}+\sigma^{-2}\Sigma_{A,A})
$ % = H(\mathbf{X}_{S})$ 
is maximized. It can be shown that this choice of $f$ is monotone submodular \cite{krause05near}. For small values of $|A|$, running \AlgLG is possible. However, we see that as the size of the active set or summary $A$ increases, the only viable option in practice is \AlgRG. 
%
%For medium-scale problems where $S$ is small, the standard greedy algorithms provide good solutions. In this Section, we show how \AlgRG can choose near-optimal subsets out of a large dataset.

\looseness -1 
In our  experiment we chose a Gaussian kernel with $h = 0.75$ and $\sigma = 1$. We used the Parkinsons Telemonitoring dataset \cite{tsanas2010enhanced} consisting of 5,875 bio-medical voice measurements with 22 attributes from people with early-stage Parkinson’s disease. We normalized the vectors to zero mean and unit norm. Fig. \ref{subfig:parkinson_utility} and \ref{subfig:parkinson_cost} compare the utility and computational cost of \AlgRG to the benchmarks for different values of $k$. 
For \AlgTG, different values of $\epsilon$ have been chosen such that a performance close to that of \AlgLG is obtained. Moreover, different values of $p$ have been chosen  such that the cost of \AlgSG is almost equal to that of \AlgRG for different values of $\epsilon$.
As we can see, \AlgRG provides the closest (practically identical) utility to that of \AlgLG with  much lower  computational cost. Decreasing the value of $\varepsilon$ results in higher utility at the price of higher computational cost.
Fig. \ref{subfig:parkinson_k=50} shows the utility versus cost of \AlgRG along with the other benchmarks for a fixed  $k = 200$ and  different values of $\epsilon$. \AlgRG provides very compelling tradeoffs between utility and cost compared to all benchmarks, including \AlgLG.
\vspace{-3mm}
\paragraph{Exemplar-based clustering.} 
A classic way to  select a set of exemplars that best represent a massive dataset is to solve the $k$-medoid problem \cite{kaufman2009finding} by minimizing the sum of pairwise dissimilarities between exemplars $A$ and elements of the dataset $V$ as follows: $L(A)=\frac{1}{V}\sum_{e \in V} \min_{v \in A} d(e,v),$
where $d \!:\! V\! \times \!V \!\rightarrow\! R$ is a distance function, encoding the dissimilarity between elements. By introducing an appropriate auxiliary element $e_0$ %(e.g., $= \underline{0}$, the all zero vector) 
we can turn $L$ into a monotone submodular function \cite{gomes10budgeted}: $f(A) = L(\{e_0\}) - L(A \cup \{e_0\}).$ Thus maximizing $f$ is equivalent to minimizing $L$ which 
%in turn 
provides a very good solution. But the problem becomes computationally challenging as the size of the summary  $A$ increases.

In our  experiment we chose $d(x, x') = ||x-x'||^2$ for the dissimilarity measure. 
%
%applied \AlgRG
%to the clustering utility with $d(x, x') = ||x-x'||^2$. 
We used a set of 10,000 Tiny Images \cite{torralba200880} where each $32 \times 32$ RGB  image was represented by a 3,072 dimensional vector. We subtracted from each vector the mean value, normalized it to unit norm, and used the origin as the auxiliary exemplar. 
Fig. \ref{subfig:images_utility} and \ref{subfig:images_cost}  compare the utility and computational cost of \AlgRG  to the benchmarks for different values of $k$.  It can be seen that \AlgRG  outperforms the benchmarks with significantly lower computational cost.
Fig. \ref{subfig:images_k=50} compares the utility versus cost of different methods for a  fixed $k = 200$ and various $p$ and $\epsilon$.  Similar to the previous experiment, \AlgRG achieves near-maximal utility at substantially lower cost compared to the other benchmarks. %We didn't show the \AlgMG in Fig. \ref{subfig:parkinson_k=50} due to its low performance.

\textbf{Large scale experiment.} We also performed a similar experiment on a larger set of 50,000 Tiny Images. For this dataset, we were not able to run \AlgLG and  \AlgFG. Hence, we compared the utility and cost of  \AlgRG  with \AlgRS using different values of $p$. As shown in Fig. \ref{subfig:images50K_utility} and Fig. \ref{subfig:images50K_cost}, \AlgRG outperforms \AlgSG in terms of both utility and cost for different values of $k$. Finally, as Fig. \ref{subfig:images50K_k=200} shows that \AlgRG achieves the highest utility but performs much faster compare to \AlgSG which is the only practical solution for this larger dataset. 

\vspace{-5mm}
\paragraph{Sensor Placement.}
When monitoring spatial phenomena, we want to deploy a limited number of sensors in an area in order to quickly detect contaminants.	
%we want to place a limited number of sensors to efficiently monitor the area.
Thus, the problem would be to select a subset of all possible locations $A \subseteq V$ to place sensors. 
Consider a set of intrusion scenarios $\mathcal{I}$ where each scenario $i \in \mathcal{I}$ defines the introduction of a contaminant at a specified point in time. For each sensor $s \in S$ and scenario $i$, the detection time, $T(s , i)$, is defined as the time it takes for  $s$ to detect $i$. If $s$ never detects  $i$, we set $T(s , i) = \infty$.  For a set of sensors $A$, detection time for scenario $i$ could be defined as  $T(A , i) = \min_{s \in A} T(s , i)$. 
Depending on the time of detection, we incur penalty $\pi_i(t)$ for detecting  scenario $i$ at time $t$.
Let $\pi_i(\infty)$ be the maximum penalty incurred if the scenario $i$ is not detected at all. Then, the penalty reduction for scenario $i$ can be defined as $R(A , i) = \pi_i(\infty) - \pi_i (T(A,i))$. 
Having a probability distribution over possible scenarios, we can calculate the expected penalty reduction for a sensor placement $A$ as $R(A)=\sum_{i \in \mathcal{I}} P(i)R(A,i)$.
This function is montone submodular \cite{krause08efficient} and for which the greedy algorithm gives us a good solution. For massive data however, we may need to resort to \AlgRG.

In our experiments we used the 12,527 node distribution network provided as part of the Battle of Water Sensor Networks (BWSN) challenge \cite{ostfeld2008battle}. 
Fig. \ref{subfig:water_imp_utility} and \ref{subfig:water_imp_cost}  compare the utility and computational cost of \AlgRG  to the benchmarks for different values of $k$.  It can be seen that \AlgRG  outperforms the benchmarks with significantly lower computational cost.
Fig. \ref{subfig:water_imp_SG_200} compares the utility versus cost of different methods for a  fixed $k = 200$ and various $p$ and $\epsilon$.  %Similar to the previous experiments,
Again \AlgRG shows similar behavior to the previous experiments by achieving near-maximal utility at much lower cost compared to the other benchmarks. 

\begin{figure*}[th!]
%\vspace{-1cm}
        \centering  
       \subfloat[Parkinsons \label{subfig:parkinson_utility}]{%
      \includegraphics[width=0.3\textwidth, height=0.265\textwidth]{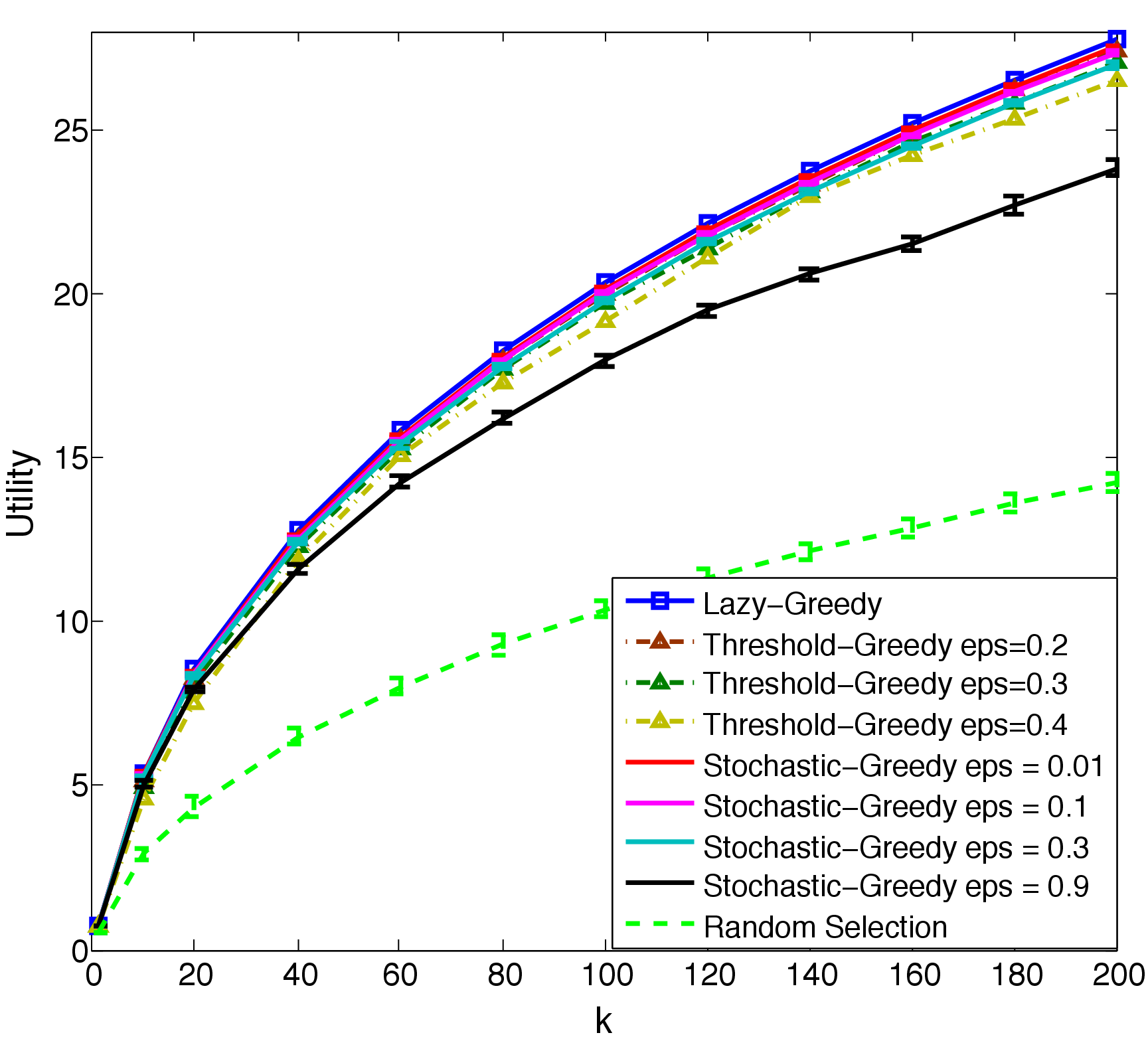}
     }
     \subfloat[Parkinsons \label{subfig:parkinson_cost}]{%
      \includegraphics[width=0.3\textwidth, height=0.265\textwidth]{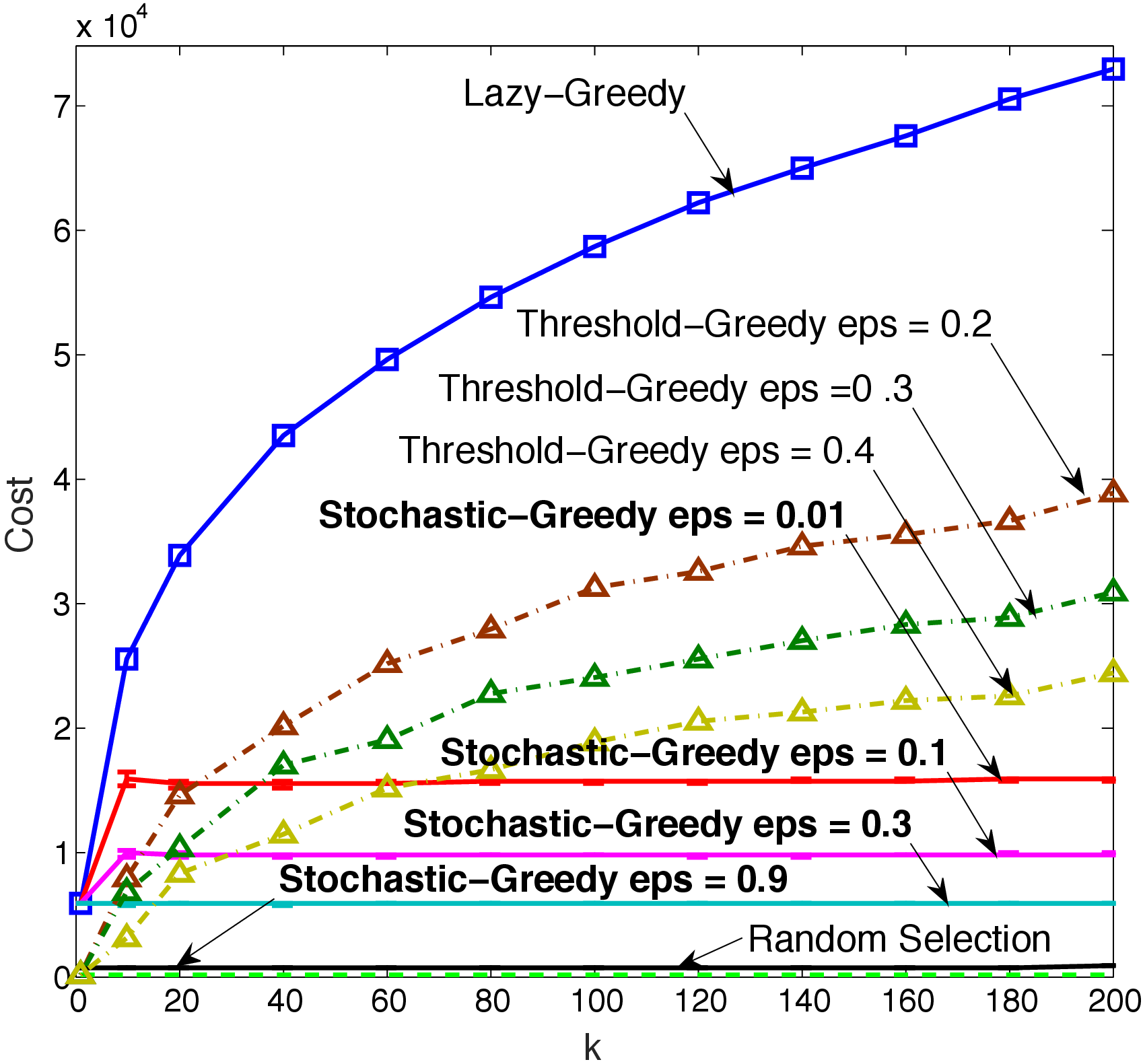}
    	 }
    	    \subfloat[Parkinsons \label{subfig:parkinson_k=50}]{%
      \includegraphics[width=0.3\textwidth, height=0.265\textwidth]{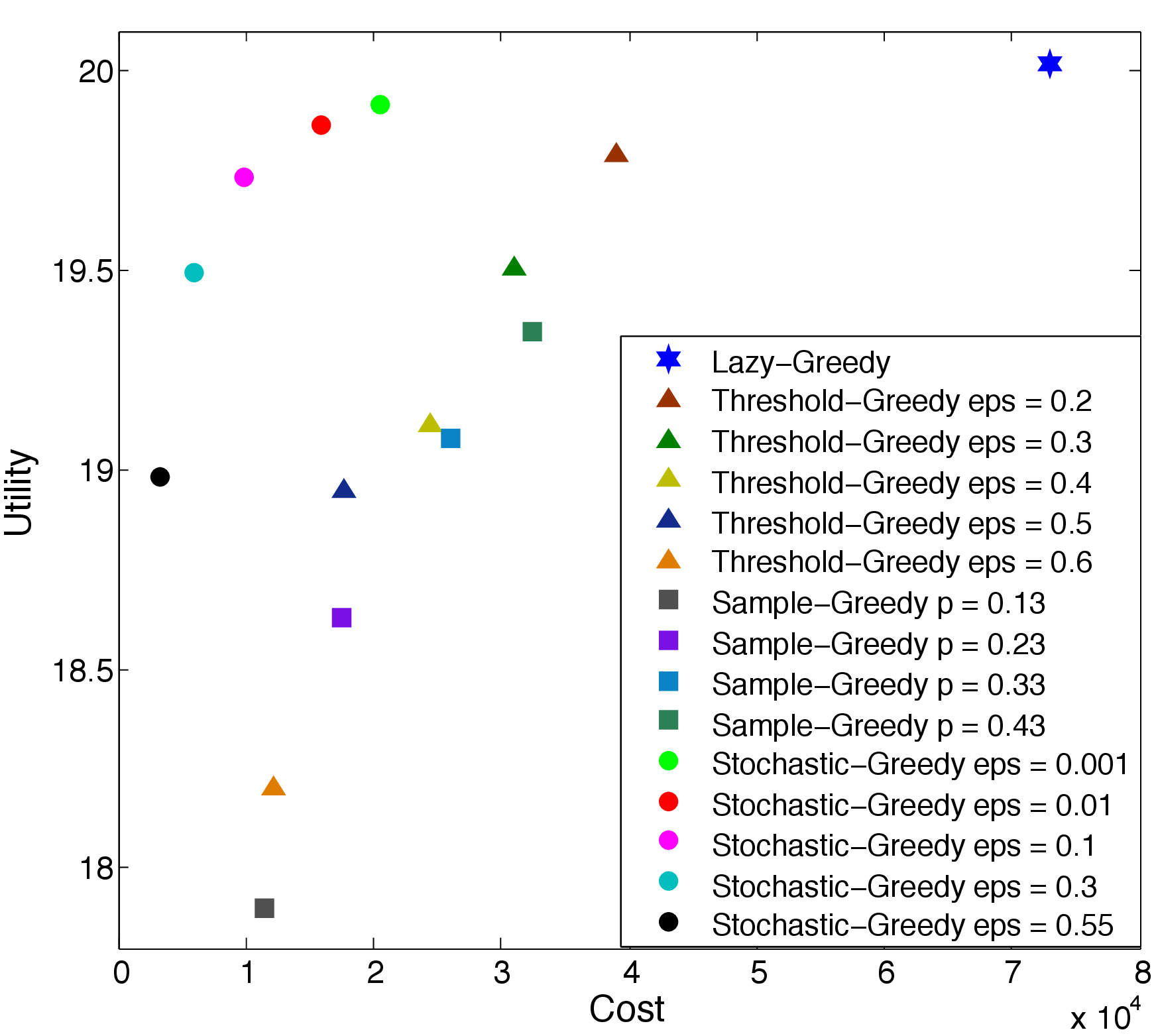}
    }
    \vspace{-4.5mm}
    \\
        \subfloat[Images 10K \label{subfig:images_utility}]{%
      \includegraphics[width=0.3\textwidth, height=0.265\textwidth]{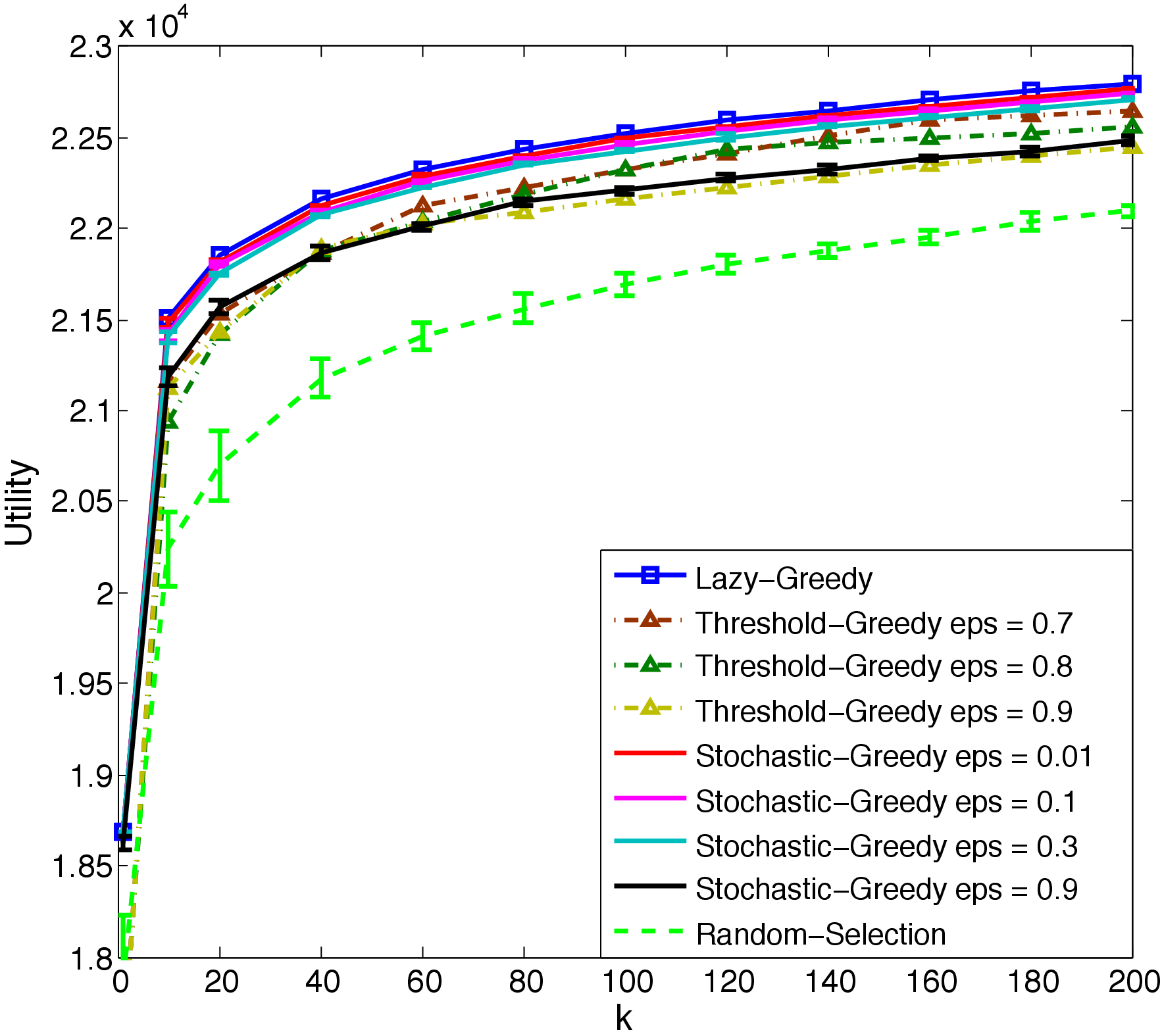}
    }
%    \hfill
    \subfloat[Images 10K \label{subfig:images_cost}]{%
      \includegraphics[width=0.3\textwidth, height=0.265\textwidth]{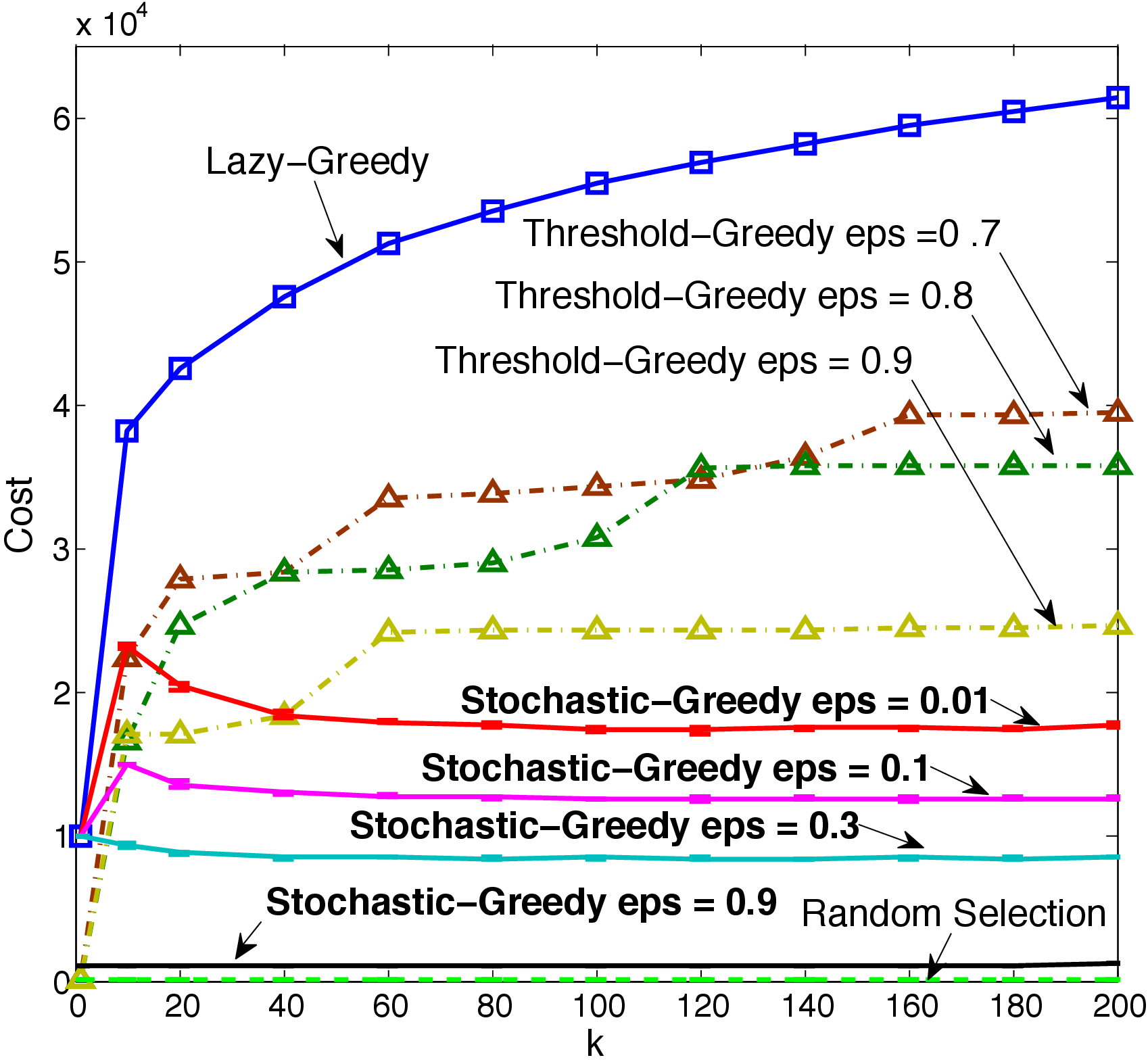}
    }
    \subfloat[Images 10K \label{subfig:images_k=50}]{%
      \includegraphics[width=0.3\textwidth, height=0.265\textwidth]{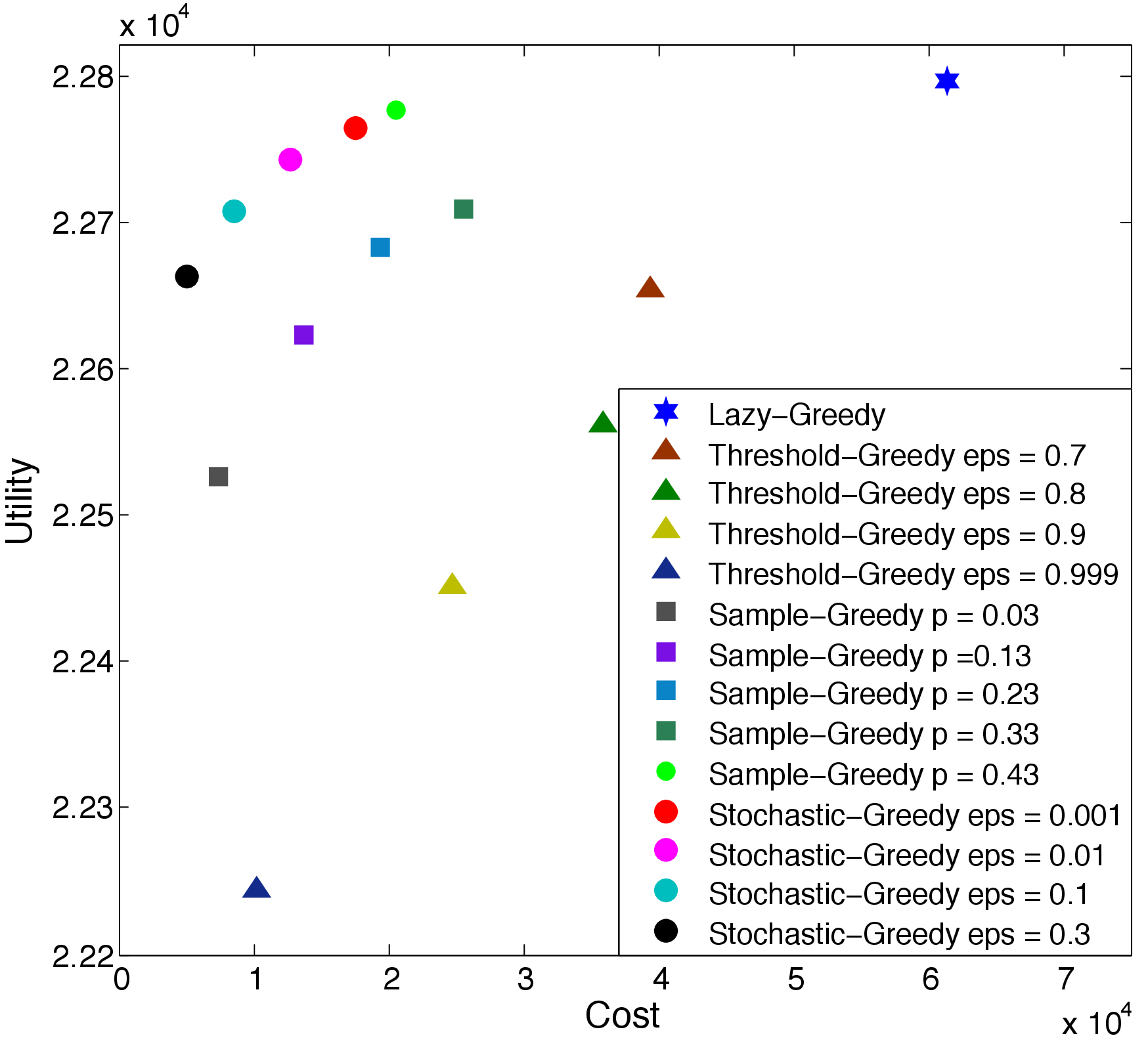}
    }
    \vspace{-4.2mm}
    \\
      \subfloat[Water Network \label{subfig:water_imp_utility}]{%
      \includegraphics[width=0.3\textwidth, height=0.265\textwidth]{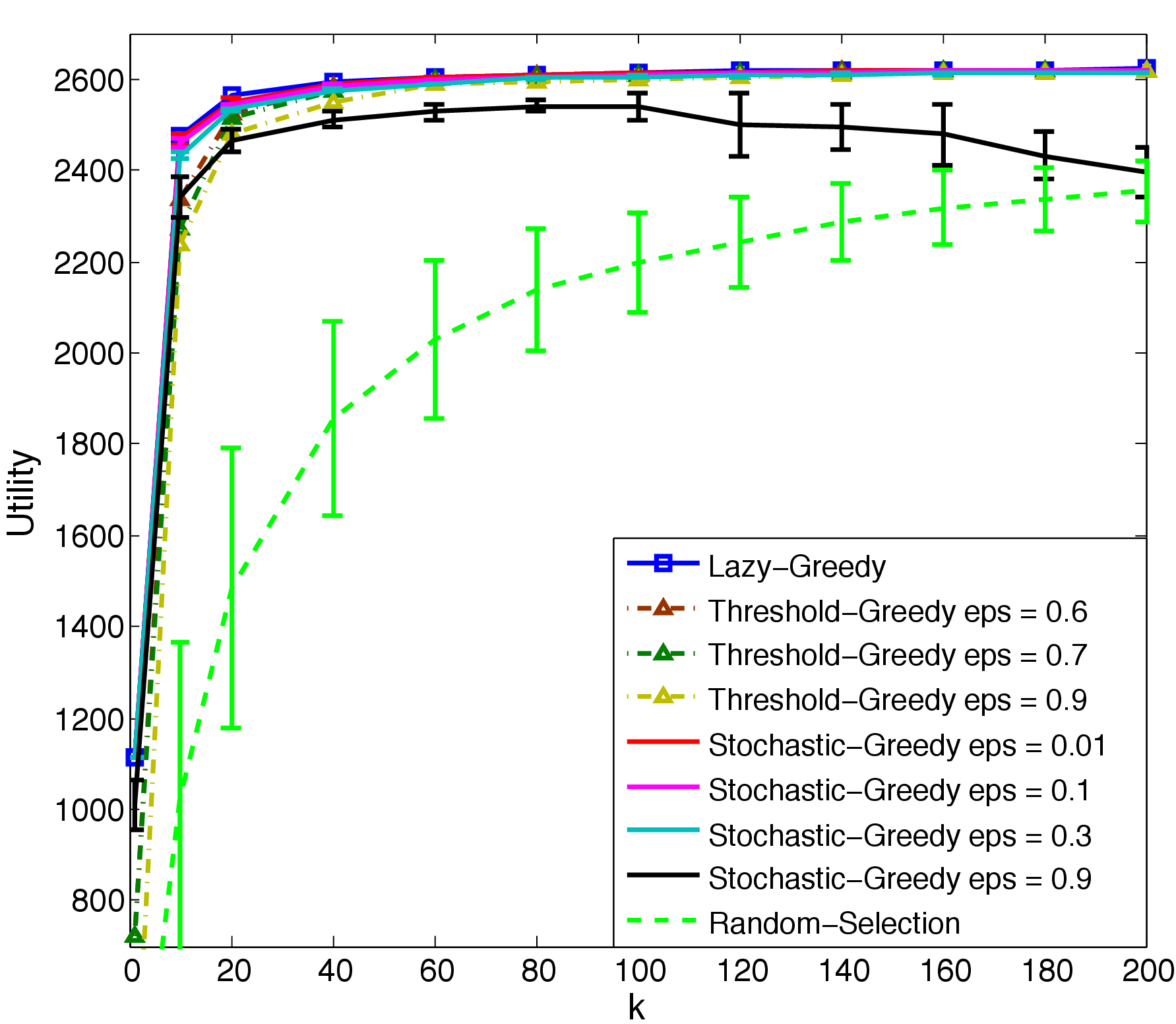}
    }
%    \hfill
    \subfloat[Water Network \label{subfig:water_imp_cost}]{%
      \includegraphics[width=0.3\textwidth, height=0.265\textwidth]{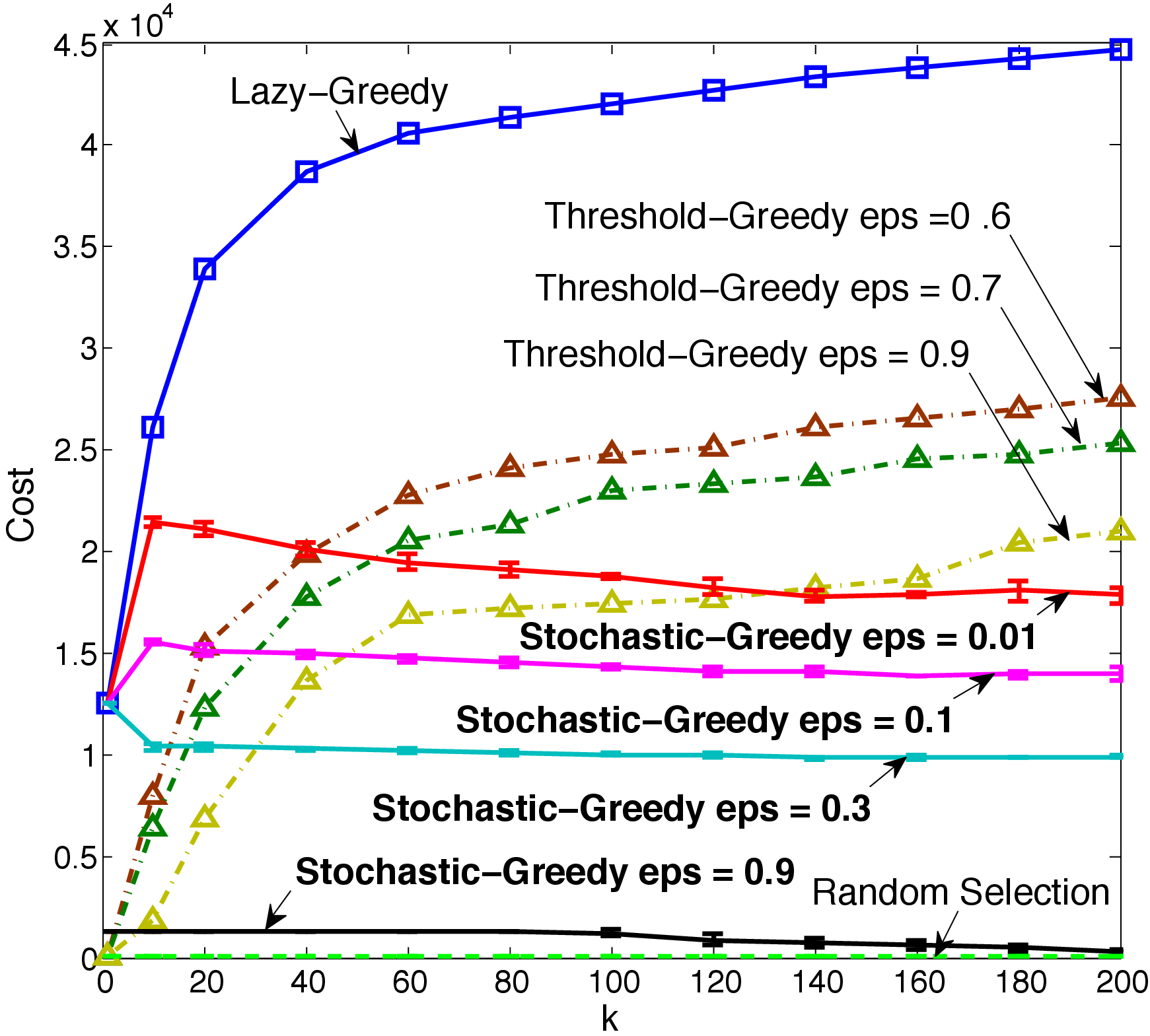}
    }
    \subfloat[Water Network \label{subfig:water_imp_SG_200}]{%
      \includegraphics[width=0.3\textwidth, height=0.265\textwidth]{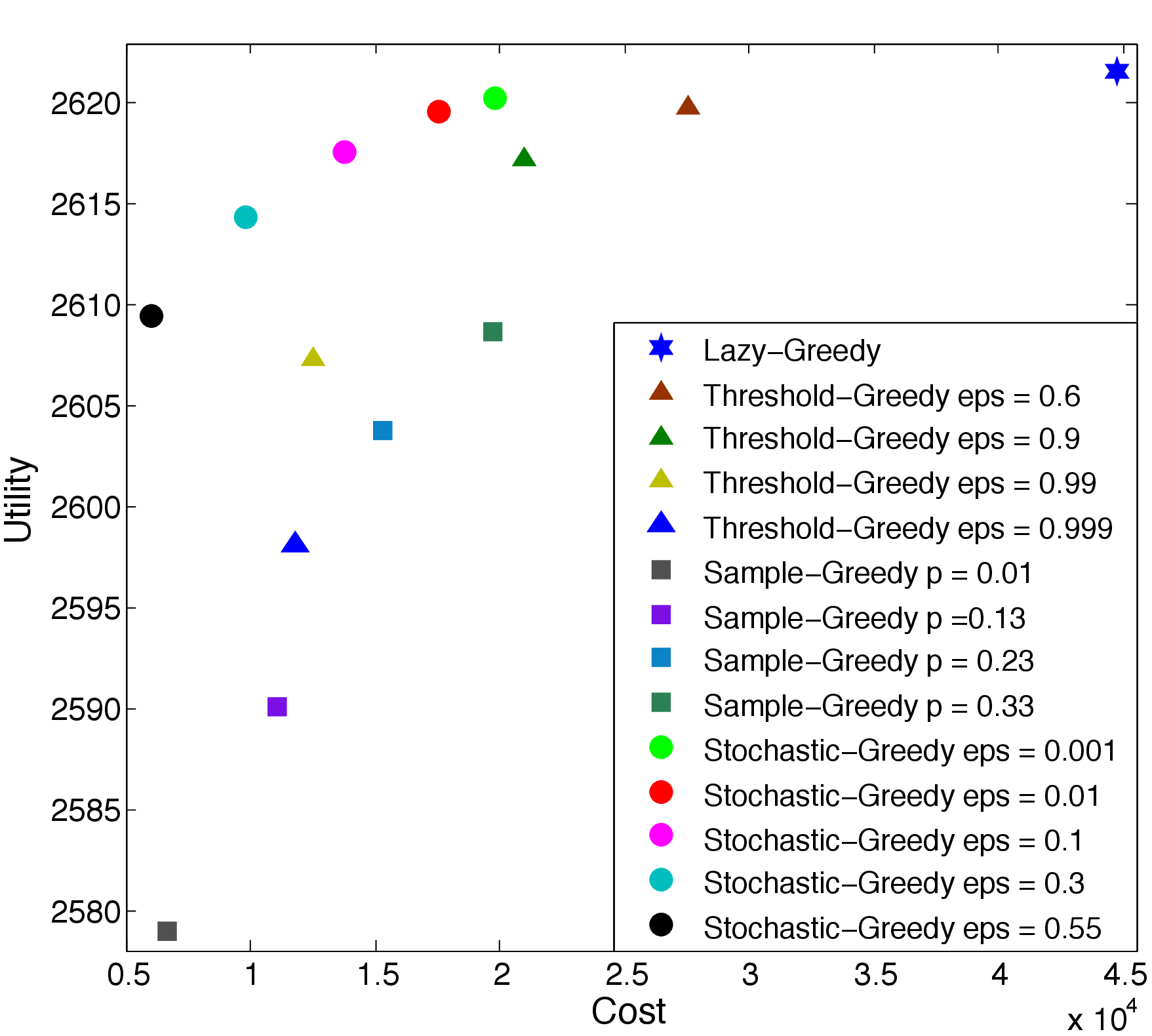}
    }
    \vspace{-4.3mm}
    \\
    %[-4mm]
    \subfloat[Images 50K \label{subfig:images50K_utility}]{%
      \includegraphics[width=0.3\textwidth, height=0.265\textwidth]{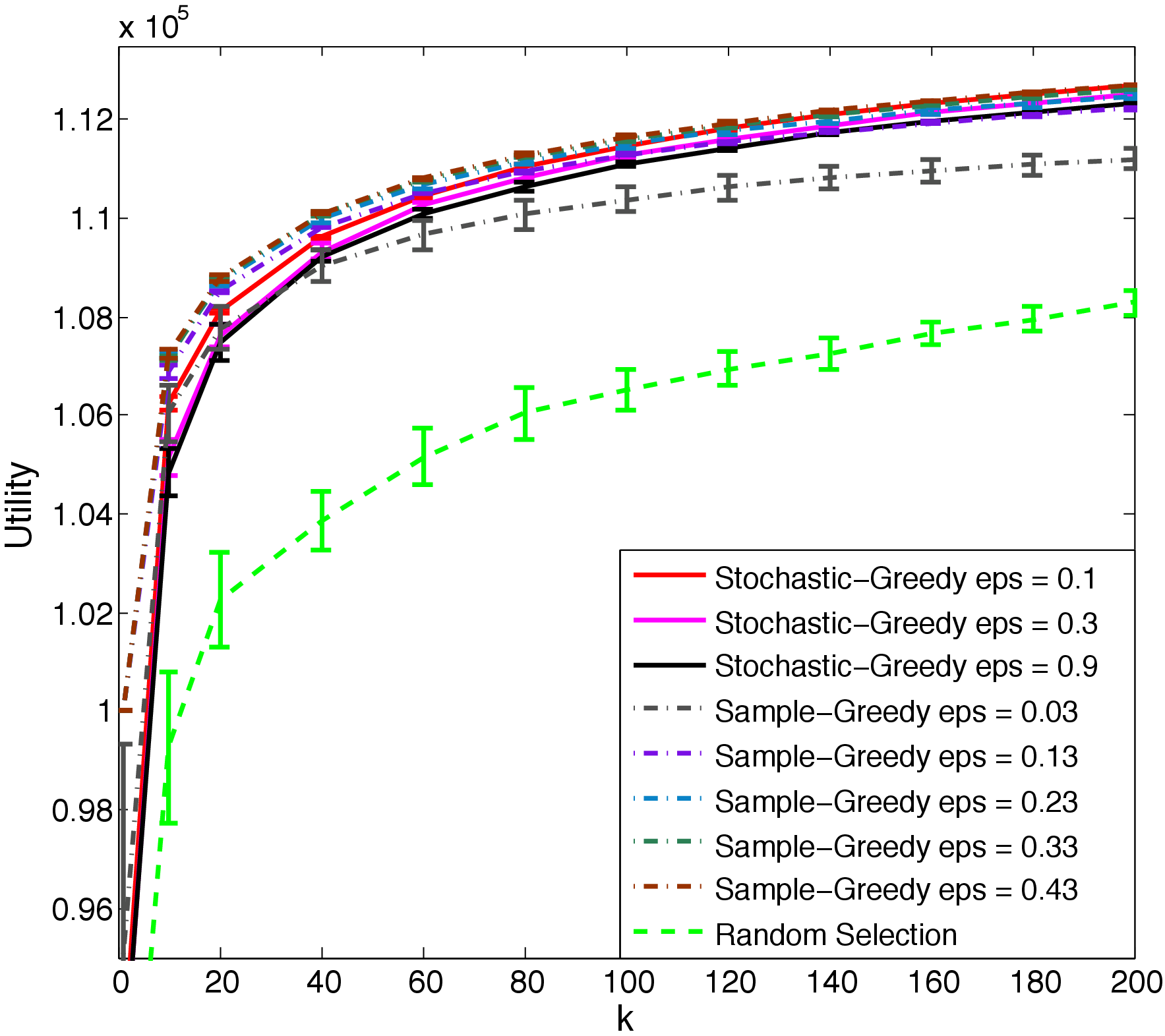}
     }
     \subfloat[Images 50K \label{subfig:images50K_cost}]{%
      \includegraphics[width=0.3\textwidth, height=0.265\textwidth]{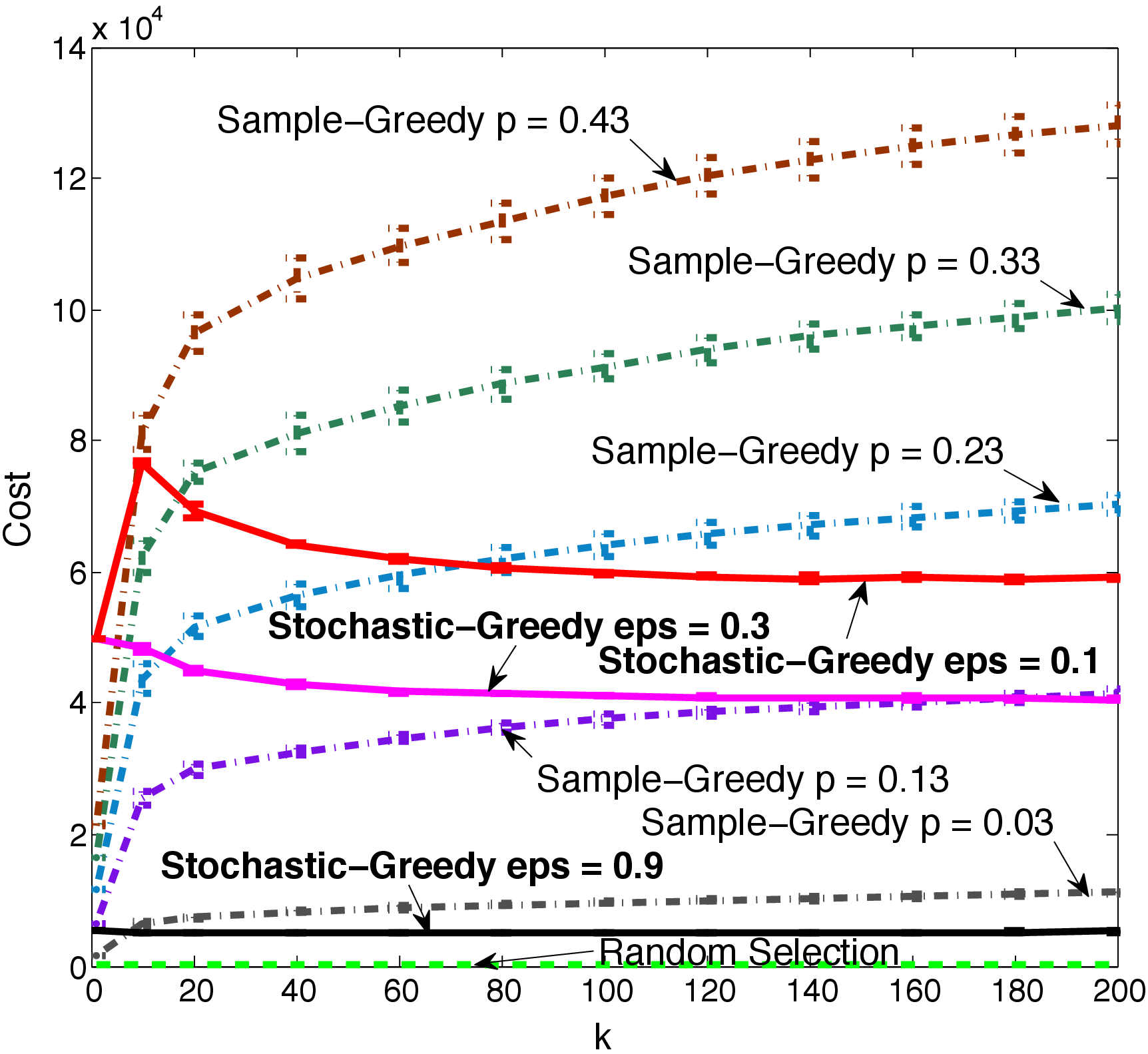}
    	 }
    	 \subfloat[Images 50K \label{subfig:images50K_k=200}]{%
      \includegraphics[width=0.3\textwidth, height=0.265\textwidth]{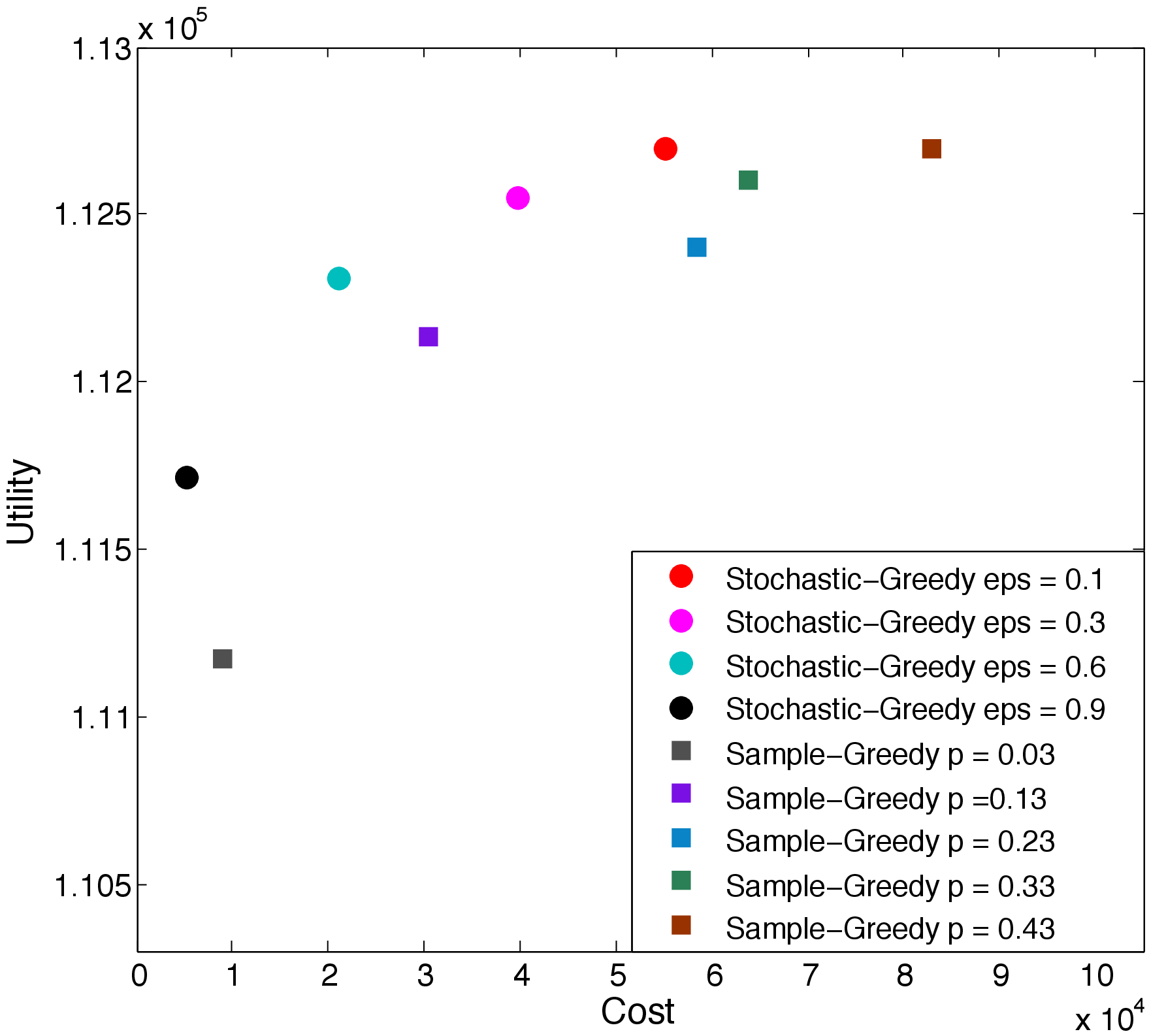}
    	 }
    	 %[-2mm]
    \vspace{-2mm}
        \caption{Performance comparisons. a), d) g) and j)  show the performance of all the algorithms for different values of $k$ on \textit{Parkinsons Telemonitoring}, a set of 10,000 \textit{Tiny Images}, \textit{Water Network}, and a set of 50,000 \textit{Tiny Images} respectively. b), e) h) and k)  show the cost of all the algorithms for different values of $k$ on the same datasets. c), f), i), l) show the utility obtained versus cost for a fixed $k = 200$. %It can be seen that \AlgRG achieves almost the same the utility with \AlgLG algorithm but is orders of magnitude faster.
        }\label{fig:exp}
    %    \vspace{-0.5cm}
    \vspace{-.2cm}
\end{figure*}

\vspace{-.2cm}
\section{Conclusion}
We have developed the first linear time algorithm \AlgRG with no dependence on $k$ for cardinality constrained submodular maximization. \AlgRG provides a $1-1/e-\epsilon$ approximation guarantee to the optimum solution with only $n \log \frac{1}{\epsilon}$ function evaluations. We have also demonstrated the effectiveness of our algorithm through an extensive set of experiments. As these show, \AlgRG achieves a major fraction of the  function utility with much less computational cost. This improvement is useful even in approaches that make use of  parallel computing or decompose the submodular function into simpler functions for faster evaluation. The properties of \AlgRG make it very appealing and necessary for solving very large scale problems. Given the importance of submodular optimization to numerous AI and machine learning applications, we believe our results provide an important step towards addressing such problems at scale. 

\section{Appendix, Analysis}
% For the approximation guarantee, we first prove the following lemma.
 The following lemma gives us the approximation guarantee.

\begin{lemma}
\label{cl:step-gain}
Given a current solution $A$, the expected gain of \AlgRG in one step is at least $\frac{1-\epsilon}{k} \sum_{a \in A^* \setminus A} \Delta(a|A)$.
\end{lemma}

\begin{proof}
Let us estimate the probability that $R \cap (A^* \setminus A) \neq \emptyset$. The set $R$ consists of $s = \frac{n}{k} \log \frac{1}{\epsilon}$ random samples from $V \setminus A$ (w.l.o.g.~with repetition), and hence
\begin{align*}
\Pr[R \cap (A^* \setminus A) = \emptyset] &=  \left(1 - \frac{|A^* \setminus A|}{|V \setminus A|} \right)^s \\[-2mm]
&\leq e^{-s \frac{|A^* \setminus A|}{|V \setminus A|}} \\[-2mm]
&\leq e^{-\frac{s}{n} |A^* \setminus A|}.\vspace{-3mm}
\end{align*}
Therefore, by using the concavity of $1 - e^{-\frac{s}{n} x}$ as a function of $x$ and the fact that $x = |A^* \setminus A| \in [0,k]$, we have%\vspace{-2mm}
%\begin{align*}
 $$ \Pr[R \cap (A^* \setminus A) \! \neq \! \emptyset] \! \geq \! 1 - e^{-\frac{s}{n} |A^* \setminus A|} \! \geq \! (1 - e^{-\frac{sk}{n}}) \frac{|A^* \setminus A|}{k}.$$
 %\end{align*}
%By using the concavity of $1 - e^{-\frac{s}{n} x}$ as a function of $x$ and the fact that $x = |A^* \setminus A| \in [0,k]$.
Recall that we chose $s = \frac{n}{k} \log \frac{1}{\epsilon}$, which gives%\vspace{-2mm}
\begin{equation}
\label{eq:hit-bound}
\Pr[R \cap (A^* \setminus A)  \neq \emptyset] \geq (1 - \epsilon) \frac{|A^* \setminus A|}{k}.
\end{equation}
%\looseness -1 
Now consider \AlgRG: it picks an element $a \in R$ maximizing the marginal value $\Delta(a|A)$. This is clearly as much as the marginal value of an element randomly chosen from $R \cap (A^* \setminus A)$ (if nonempty). Overall, $R$ is equally likely to contain each element of $A^* \setminus A$, so a uniformly random element of $R \cap (A^* \setminus A)$ is actually a uniformly random element of $A^* \setminus A$. Thus, we obtain
$$ \E[\Delta(a|A)] \!\geq\! \Pr[R \cap (A^* \setminus A) \neq \emptyset] \times \frac{1}{|A^* \setminus A|} \!\!\sum_{a \in A^* \setminus A} \!\!\!\Delta(a|A).$$
Using (\ref{eq:hit-bound}), we conclude that $ \E[\Delta(a|A)] \geq \frac{1-\epsilon}{k}  \sum_{a \in A^* \setminus A}  \Delta(a|A)$.
\end{proof}
Now it is straightforward to finish the proof of Theorem~\ref{thm:cardinality}. Let $A_i = \{a_1,\ldots,a_i\}$ denote the solution returned by \AlgRG after  $i$ steps.
% Condition on a solution $A_i = \{a_1,\ldots,a_i\}$ after $i$ steps.
From Lemma~\ref{cl:step-gain},  $$\E[\Delta(a_{i+1}|A_i) \mid A_i] \geq \frac{1-\epsilon}{k} \sum_{a \in A^* \setminus A_i} \Delta(a|A_i).\vspace{-4mm}$$
By submodularity, 
$$\sum_{a \in A^* \setminus A_i} \Delta(a|A_i) \geq \Delta(A^*|A_i) \geq f(A^*) - f(A_i).$$ Therefore,
\begin{align*}
\E[f(A_{i+1}) - f(A_i) \mid A_i] &= \E[\Delta(a_{i+1}|A_i) \mid A_i] \\
&\geq \frac{1-\epsilon}{k} (f(A^*) - f(A_i)).
\end{align*}
By taking expectation over $A_i$,
$$ \E[f(A_{i+1}) - f(A_i)] \geq \frac{1-\epsilon}{k} \E[f(A^*) - f(A_i)].$$
By induction, this implies that
\begin{align*}
\E[f(A_k)] &\geq \left( 1 - \left(1 - \frac{1-\epsilon}{k} \right)^k \right) f(A^*) \\
 &\geq \left(1 - e^{-(1-\epsilon)}\right) f(A^*) 
 \geq (1 - 1/e - \epsilon) f(A^*).
\end{align*}

\paragraph{Acknowledgment.} This research was supported by SNF 200021-137971, ERC StG 307036, a Microsoft Faculty Fellowship, and an ETH Fellowship.

\clearpage
\bibliographystyle{aaai}
\bibliography{Mirzasoleiman}

\end{document}